\documentclass[sigconf, nonacm]{acmart}
\AtBeginDocument{%
  }

\setcopyright{acmlicensed}
\copyrightyear{2025}
\acmYear{2025}
\acmDOI{XXXXXXX.XXXXXXX}
\acmConference[Conference acronym 'XX]{Make sure to enter the correct
  conference title from your rights confirmation email}{June 03--05,
  2018}{Woodstock, NY}
\acmISBN{978-1-4503-XXXX-X/2018/06}

\usepackage{mathtools}
\mathtoolsset{showonlyrefs=true}

\usepackage{booktabs, array}
\newcolumntype{P}[1]{>{\centering\arraybackslash}p{#1}}
\newcolumntype{M}[1]{>{\centering\arraybackslash}m{#1}}
\newcolumntype{L}[1]{>{\raggedright\arraybackslash}m{#1}}

\usepackage{multirow}

\usepackage{subcaption}

\usepackage[ruled,vlined,linesnumbered]{algorithm2e}

\usepackage{tcolorbox}

\newtcolorbox[auto counter]{mybox}[3][]{colframe = #2!25,
  colback  = #2!10,
  coltitle = #2!20!black, float, title={Box~\thetcbcounter: #3},#1}

\usepackage{graphicx}
\usepackage{tikz}
\usetikzlibrary{arrows.meta,positioning,fit,calc,backgrounds,shapes.geometric,decorations.pathreplacing}

\definecolor{train}{HTML}{E8F5E9}     
\definecolor{serve}{HTML}{E3F2FD}     
\definecolor{analysis}{HTML}{FFF3E0}  

\begin{document}

\title[Experimentation Accelerator]{Experimentation Accelerator: Interpretable Insights and Creative Recommendations for A/B Testing with Content-Aware ranking}
 
\author{Zhengmian Hu}
\authornote{Both authors contributed equally to this research.}
\affiliation{%
  \institution{Adobe Research}
  \city{San Jose}
  \state{CA}
  \country{USA}
}
\email{zhengmianh@adobe.com}

\author{Lei Shi}
\authornotemark[1]
\affiliation{%
  \institution{Adobe Research}
  \city{San Jose}
  \state{CA}
  \country{USA}
}
\email{leis@adobe.com}

\author{Ritwik Sinha}
\affiliation{%
  \institution{Adobe Research}
  \city{San Jose}
  \state{CA}
  \country{USA}
}
\email{risinha@adobe.com}

\author{Justin Grover}
\affiliation{%
  \institution{Adobe Digital Experience}
  \city{San Jose}
  \state{CA}
  \country{USA}
}
\email{jgrover@adobe.com}

\author{David Arbour}
\affiliation{%
  \institution{Adobe Research}
  \city{San Jose}
  \state{CA}
  \country{USA}
}
\email{arbour@adobe.com}







\renewcommand{\shortauthors}{Trovato et al.}

\begin{abstract}
  Modern online experimentation faces two bottlenecks: scarce traffic forces tough choices on which variants to test, and post-hoc insight extraction is manual, inconsistent, and often content-agnostic. Meanwhile, organizations underuse historical A/B results and rich content embeddings that could guide prioritization and creative iteration. We present a unified framework to (i) prioritize which variants to test, (ii) explain why winners win, and (iii) surface targeted opportunities for new, higher-potential variants. Leveraging treatment embeddings and historical outcomes, we train a CTR ranking model with fixed effects for contextual shifts that scores candidates while balancing value and content diversity. For better interpretability and understanding, we project treatments onto curated semantic marketing attributes and re-express the ranker in this space via a sign-consistent, sparse constrained Lasso, yielding per-attribute coefficients and signed contributions for visual explanations, top-k drivers, and natural-language insights. We then compute an opportunity index combining attribute importance (from the ranker) with under-expression in the current experiment to flag missing, high-impact attributes. Finally, LLMs translate ranked opportunities into concrete creative suggestions and estimate both learning and conversion potential, enabling faster, more informative, and more efficient test cycles. These components have been built into a real Adobe product, called \textit{Experimentation Accelerator}, to provide AI-based insights and opportunities to scale experimentation for customers. We provide an evaluation of the performance of the proposed framework on some real-world experiments by Adobe business customers that validate the high quality of the generation pipeline. 
\end{abstract}

\ccsdesc[500]{Computing methodologies~Natural language generation}
\ccsdesc[500]{Human-centered computing~User studies}
\ccsdesc[500]{Theory of computation~Structured prediction}
\ccsdesc[500]{Computing methodologies~Principal component analysis}
\ccsdesc[500]{Information systems~Learning to rank}
\ccsdesc[500]{Information systems~Sentiment analysis}

\keywords{Transfer learning; GenAI-powered experimentation; Learning to rank; Interpretability}


\maketitle

\section{Introduction}
Online randomized experimentation underpins modern digital marketing: it evaluates product changes, personalizes experiences at scale, and turns customer interactions into measurable uplift in engagement and revenue. It provides unbiased estimation and uncertainty quantification of the effect of the change, facilitating causal understanding, solid business insight, and better decision-making. 

Nowadays, however, as experiments grow to larger scales, the mechanics that make experimentation reliable begin to encounter practical limits. In this paper, we investigate two persistent constraints that become more acute at scale. First, the number of variants to be tested is increasing, especially boosted by the success of Generative AI. Each additional variant dilutes statistical power, raises the minimum detectable effect, and lengthens test duration. This forces hard choices about what to launch under fixed user budgets and campaign timelines, making naive ``test everything'' strategies infeasible. Second, two critical downstream tasks are typically ad hoc: \emph{insight extraction} and \emph{opportunity generation}. By \emph{insight extraction}, we mean converting raw A/B outcomes and variant content into business insights—e.g., identifying which semantic attributes (tone, incentive, urgency) explain observed performance gaps. By \emph{opportunity generation}, we mean systematically proposing improvement in potential variants and turning insight into actionable creative changes. When done manually, both tasks are costly across many concurrent tests and prone to analyst-specific bias.

Harnessing recent advances in generative AI and representation learning, we introduce a unified framework that exploits two high-value assets to address the above challenges: archives of historical A/B results (with effect sizes and uncertainties) and rich content embeddings that capture the semantic structure of creatives. Concretely, our approach (i) prioritizes which variants to test, (ii) explains why certain variants win, and (iii) reveals targeted opportunities for new, higher-potential variants. The framework proceeds in three integrated stages. \textbf{Ranking.} We train a click-through rate (CTR) ranking model on the embeddings of historical treatments and outcomes. At inference time, the model can predict the relative rank of candidate variants in a new experiment. \textbf{Insights.} For better interpretability and deeper understanding, we project treatments onto a curated set of semantic marketing attributes to obtain attribute scores, then re-express the ranker in this attribute space. This yields per-attribute coefficients that quantify the contribution of each attribute in driving predicted CTR gaps. It further enables data-driven natural-language insights generation. \textbf{Opportunities.} We construct an opportunity index that combines attribute importance (from the ranker) with attribute under-expression in the current experiment to detect missing, high-impact attributes; large language models (LLMs) translate these ranked opportunities into concrete creative suggestions and estimate both learning and conversion potential.

Our contributions are threefold:
\begin{enumerate}
    \item A \emph{ranking model} that transfers the knowledge from historical experiment through embeddings, that can predict the relative rank of candidate variants. This step leverages information from high-quality A/B testing datasets and builds a foundation with solid statistical rigor for further downstream tasks. 
    
    \item An \emph{insight extraction module} that projects the ranker model weights into a semantic attribute space via constrained Lasso, producing signed per-attribute contributions and standardized insight artifacts (visuals, top-$k$ drivers, natural-language rationales).
    
    \item An \emph{opportunity generation module} that is built on a principled opportunity index that surfaces missing, high-impact attributes and turns them into LLM-generated creative improvement suggestions.
\end{enumerate}

These components are implemented in a production system, \textit{Experimentation Accelerator}, now live with Adobe customers\footnote{See the official website \url{https://business.adobe.com/products/journey-optimizer/experimentation-accelerator.html}}. We evaluate the framework on real-world experiments, showing that the pipeline can provide satisfactory and faithful explanations tied to crucial marketing attributes, and actionable opportunities that accelerate more informative and efficient test cycles.

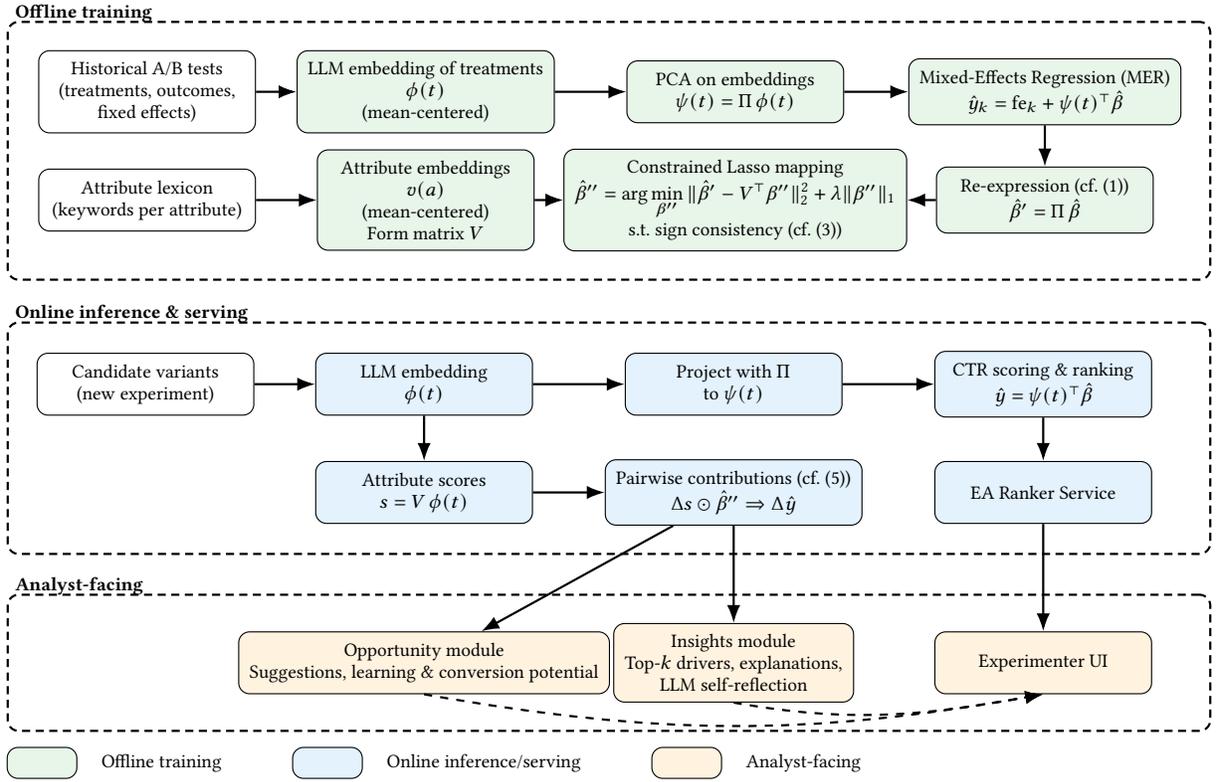
\begin{figure*}[t]
\centering
\resizebox{0.9\textwidth}{!}{%
\begin{tikzpicture}[
  node distance = 9mm and 13mm,
  >=Latex,
  every node/.style = {font=\footnotesize},
  group/.style     = {draw, rounded corners, thick, inner sep=10pt, densely dashed},
  swatch/.style    = {draw, minimum width=9mm, minimum height=4mm},
  datablock/.style = {draw, rounded corners, align=center, inner sep=4pt, minimum height=8mm, minimum width=28mm, fill=white},
  proc/.style      = {datablock, fill=train},
  serv/.style      = {datablock, fill=serve},
  ana/.style       = {datablock, fill=analysis},
  arr/.style       = {->, thick},
  art/.style       = {->, thick, dashed}
]

\newlength{\laneFrameSep}   
\setlength{\laneFrameSep}{10mm}
\newlength{\groupInnerSep}  
\setlength{\groupInnerSep}{10pt}
\newlength{\laneContentSep} 
\setlength{\laneContentSep}{\laneFrameSep}
\addtolength{\laneContentSep}{2\groupInnerSep}

\coordinate (offTL) at (10mm,-8mm);
\coordinate (c11) at (offTL);
\coordinate (c12) at ($(offTL)+(40mm,0)$);
\coordinate (c13) at ($(offTL)+(80mm,0)$);
\coordinate (c14) at ($(offTL)+(120mm,0)$);
\coordinate (c21) at ($(offTL)+(0,-14mm)$);
\coordinate (c22) at ($(offTL)+(40mm,-14mm)$);
\coordinate (c23) at ($(offTL)+(80mm,-14mm)$);
\coordinate (c24) at ($(offTL)+(120mm,-14mm)$);

\node[datablock, anchor=center] (hist2)      at ($(c11)+(4mm,0)$) {Historical A/B tests\\(treatments, outcomes,\\ fixed effects)};
\node[proc,      anchor=center] (embTtrain2) at (c12) {LLM embedding of treatments\\$\phi(t)$ \\ (mean-centered)};
\node[proc,      anchor=center] (pca2)       at (c13) {PCA on embeddings\\$\psi(t)=\Pi\,\phi(t)$};
\node[proc,      anchor=center] (mer2)       at (c14) {Mixed-Effects Regression (MER)\\$\hat{y}_k=\mathrm{fe}_k+\psi(t)^\top\hat{\beta}$};
\node[datablock, anchor=center] (lexicon2)   at ($(c21)+(4mm,0)$) {Attribute lexicon\\(keywords per attribute)};
\node[proc,      anchor=center] (attremb2)   at (c22) {Attribute embeddings\\$v(a)$\\(mean-centered)\\Form matrix $V$};
\node[proc,      anchor=center] (lasso2)     at (c23) {Constrained Lasso mapping\\$\displaystyle \hat{\beta}''=\arg\min_{\beta''}\|\hat{\beta}'-V^\top\beta''\|_2^2+\lambda\|\beta''\|_1$\\s.t. sign consistency (cf. \eqref{eq:ranker_model_decomp})};
\node[proc,      anchor=center] (betaprime2) at (c24) {Re-expression (cf. \eqref{eq:ranker_model})\\$\hat{\beta}'=\Pi\,\hat{\beta}$};

\node[group, draw=none, inner sep=0,
      fit=(hist2)(embTtrain2)(pca2)(mer2)(lexicon2)(attremb2)(lasso2)(betaprime2)] (gtrain2box) {};

\coordinate (serveTL) at ([xshift=10mm,yshift=-\laneContentSep]gtrain2box.south west);
\coordinate (s11) at (serveTL);
\coordinate (s12) at ($(serveTL)+(40mm,0)$);
\coordinate (s13) at ($(serveTL)+(80mm,0)$);
\coordinate (s14) at ($(serveTL)+(120mm,0)$);
\coordinate (s21) at ($(serveTL)+(0,-14mm)$);
\coordinate (s22) at ($(serveTL)+(40mm,-14mm)$);
\coordinate (s23) at ($(serveTL)+(80mm,-14mm)$);
\coordinate (s24) at ($(serveTL)+(120mm,-14mm)$);

\node[datablock, anchor=center] (cands2)    at ($(s11)+(4mm,0)$) {Candidate variants\\(new experiment)};
\node[serv,      anchor=center] (embTinfer2) at (s12) {LLM embedding\\$\phi(t)$};
\node[serv,      anchor=center] (proj2)      at (s13) {Project with $\Pi$\\to $\psi(t)$};
\node[serv,      anchor=center] (predict2)   at (s14) {CTR scoring \& ranking\\$\hat{y}=\psi(t)^\top\hat{\beta}$};

\node[serv,      anchor=center] (scores2)    at (s22) {Attribute scores\\$s=V\,\phi(t)$};
\node[serv,      anchor=center] (contrib2)   at (s23) {Pairwise contributions (cf. \eqref{eqn:pairwise_contribution})\\$\Delta s \odot \hat{\beta}'' \Rightarrow \Delta \hat{y}$};
\node[serv,      anchor=center] (service2)   at (s24) {EA Ranker Service};

\node[group, draw=none, inner sep=0,
      fit=(cands2)(embTinfer2)(proj2)(predict2)(scores2)(contrib2)(service2)] (gserve2box) {};

\coordinate (anaTL) at ($(serveTL)+(0,-19mm-\laneContentSep)$);
\coordinate (a11) at ($(anaTL)+(40mm,0)$);
\coordinate (a12) at ($(anaTL)+(80mm,0)$);
\coordinate (a13) at ($(anaTL)+(120mm,0)$);

\node[ana, anchor=center] (ui2)       at (a13) {Experimenter UI};
\node[ana, anchor=center] (insights2) at (a12) {Insights module\\Top-$k$ drivers, explanations,\\LLM self-reflection};
\node[ana, anchor=center] (opps2)     at (a11) {Opportunity module\\Suggestions, learning \& conversion potential};

\begin{scope}[on background layer]
  \node[group, draw=none, inner sep=0,
        fit=(ui2)(insights2)(opps2)] (ganalytics2box) {};

  \node[inner sep=0, draw=none, fit=(gtrain2box)(gserve2box)(ganalytics2box)] (g2bounds) {};

  \node[group, label={[anchor=west,yshift=1mm]north west:\textbf{Offline training}},
        fit=(g2bounds.west |- gtrain2box.north) (g2bounds.east |- gtrain2box.south)] (gtrain2aligned) {};
  \node[group, label={[anchor=west,yshift=1mm]north west:\textbf{Online inference \& serving}},
        fit=(g2bounds.west |- gserve2box.north) (g2bounds.east |- gserve2box.south)] (gserve2aligned) {};
  \node[group, label={[anchor=west,yshift=1mm]north west:\textbf{Analyst-facing}},
        fit=(g2bounds.west |- ganalytics2box.north) (g2bounds.east |- ganalytics2box.south)] (ganalytics2aligned) {};
\end{scope}

\draw[arr] (hist2) -- (embTtrain2);
\draw[arr] (embTtrain2) -- (pca2);
\draw[arr] (pca2) -- (mer2);
\draw[arr] (mer2) -- (betaprime2);
\draw[arr] (lexicon2) -- (attremb2);
\draw[arr] (betaprime2) -- (lasso2);
\draw[arr] (attremb2) -- (lasso2);

\draw[arr] (cands2) -- (embTinfer2);
\draw[arr] (embTinfer2) -- (proj2);
\draw[arr] (proj2) -- (predict2);
\draw[arr] (predict2) -- (service2);
\draw[arr] (embTinfer2) -- (scores2);
\draw[arr] (scores2) -- (contrib2);


\draw[arr] (service2.south) -- (ui2.north);
\draw[arr] (contrib2) -- (insights2);
\draw[arr] (contrib2) -- (opps2);

\draw[art] (insights2.south) to[out=-10, in=-170] (ui2.south);
\draw[art] (opps2.south) to[out=-10, in=-170] (ui2.south);

\node[proc, minimum width=9mm, minimum height=4mm, below=4mm of ganalytics2aligned.south west, anchor=west] (legA) {};
\node[anchor=west] (legAlabel) at ([xshift=2mm]legA.east) {Offline training};

\node[serv, minimum width=9mm, minimum height=4mm, right=8mm of legAlabel] (legB) {};
\node[anchor=west] (legBlabel) at ([xshift=2mm]legB.east) {Online inference/serving};

\node[ana, minimum width=9mm, minimum height=4mm, right=8mm of legBlabel] (legC) {};
\node[anchor=west] (legClabel) at ([xshift=2mm]legC.east) {Analyst-facing};

\end{tikzpicture}%
}
\caption{Overall architecture of the Experimentation Accelerator: (green) offline training learns the PCA projection $\Pi$, MER weights $\hat{\beta}$, and attribute-space coefficients $\hat{\beta}''$; (blue) online inference reuses $(\Pi,\hat{\beta},V,\hat{\beta}'')$ to score candidates, attribute-score them, and compute pairwise contributions; (orange) analyst-facing modules turn quantitative attributions into insights and opportunity suggestions. See \eqref{eq:ranker_model}-\eqref{eq:ranker_model_decomp} for the model re-expression.}
\label{fig:simple-scaffold}
\end{figure*}

\section{Related work}

\paragraph{Content-aware ranking.} Related work includes the following \cite{countbayesie2023gptabtest,ellickson2023treatmentembeddings}, which explores the idea of using treatment embeddings to learn a prediction model for the outcome or rank of several variants. A key assumption in these works is that there is a large dataset of experiments that can be readily used to predict the relative ranking of treatments, and evaluation is often on random hold-out samples. In reality, many brands operate in data-sparse settings and require learning from completed experiments in a different domain, raising the question of whether such learnings will carry over.

\paragraph{CTR prediction and reaction modeling.} CTR prediction for ads focuses primarily on predictive accuracy \cite{liu2023prior, ye2024lolallmassistedonlinelearning, acm3219819, shi2025leveraging}. Reaction modeling for social media posts also aims to learn the community's reaction from past posts and provide predictive models before publication \cite{acm2512875}. In contrast, our method not only predicts the ranking of treatments but also provides interpretable insights and opportunities through marketing attributes and ranking model coefficients. This further helps with test prioritization and creative improvement in future experiments. 

\paragraph{LLM as a judge.} Prompting an LLM or using LLM-as-a-judge to perform text generation evaluation has also been explored, e.g., \cite{arxiv2302.04166}. This approach does not associate treatments with outcome data in real-world A/B testing datasets. Our method instead derives insights from quantitative models that are trained for our specific experimentation tasks.

\paragraph{Content optimization and creative generation.} Related works also include content optimization and creative generation \cite{lee2024calm,columbia2020creatives,letham2019noisybo}. \cite{lee2024calm} performs language-model fine-tuning for content optimization. \cite{columbia2020creatives} develops a ranking model as well as a deep-learning-based generation model. \cite{letham2019noisybo} applies Bayesian optimization to maximize treatment rewards. However, these works typically operate directly in the embedding space rather than at the level of marketing attributes, making it less explainable which attributes are most important from historical experiments and which attributes need to be improved. Our approach achieves this goal in a more interpretable manner.

\section{Semantic ranking model}
\label{sec:semantic-ranking}
In this section, we introduce the technical details of the core pipeline of the backbone ranking model for our experimentation accelerator. The high-level question is, given a set of variants in an experiment, can we predict their final ranking in terms of their prospective outcomes? This requires building a model based on some historical experimental data that can explore the semantic information underneath the treatment arms, and transfer that knowledge to the new experiments. In our approach, we focus on ranking the treatment arms based on CTR:
\begin{align}
    \text{CTR} = \frac{\text{Number of Clicks}}{\text{Number of Impressions}} \times 100\%. 
\end{align}

\subsection{Choice of training and testing data}
The first question to consider is what type of training data to use for building such a ranking model, and what data to use to evaluate the model. There are multiple concerns:

First, we need datasets in digital marketing scenarios, instead of a general text dataset for other purposes. It is very crucial to ensure the dataset satisfies good alignment with our business purpose is very crucial. Different application scenarios may collect very different text materials, and even the same text material can be interpreted differently in different settings. This rules out many open benchmark datasets for general text analysis applications, such as classification or semantic clustering. 
   
Second, we need to use A/B testing datasets that align well with the modality of our experiments with a similar target outcome. The importance of A/B testing data is highlighted in its solid randomization, which guarantees that there are no unmeasured confounders in outcome prediction. Working with such datasets frees the model from hidden bias that can be induced by factors that are orthogonal to the semantic information of the texts. Moreover, targeting on the similar outcome guarantees that the knowledge can be transferred and interpreted plausibly. 

Third, we need to utilize public datasets for training and our direct customer data for testing. Our goal here is to build a platform that provides business-to-business (B2B) service to other companies for better solutions to conduct A/B testing. Although we can collect many customers' data, these cannot be used as training data because of the legal restrictions on customer data usage. Nevertheless, these datasets will serve as perfect testing sources as the evaluations will directly target our customers. This will provide more solid evidence than other approaches, such as train/test splitting-based evaluation, and validate the quality of transfer learning. 

With these concerns, we use the Upworthy dataset \citep{matias2021upworthy} as our training data, which is an open dataset of thousands of A/B tests of headlines conducted by Upworthy from January 2013 to April 2015. It contains the CTR for the headlines, which aligns perfectly with our target outcome. For testing data, we include $65$ experiments which are collected from our customers.

\subsection{Ranker model construction} 
In this section, we present the full ranking model construction pipeline. 
Let $K_{\text{train}}$ and $K_{\text{test}}$ denote the number of experiments in the training data and testing data, respectively, and $t_{k,i}$ denote the $i$-th treatment in $k$-th experiment. The pipeline is summarized in Algorithm \ref{alg:ranking}.

\begin{algorithm}[ht]
\small 
\caption{Text-Embedding-Based Mixed-Effects Regression for CTR Prediction}
\label{alg:ranking}
\KwIn{Marketing copies $\{t_{k,i}\}$ across experiments $k$; embedding model $\phi(\cdot)$; target dimension $q$}
\KwOut{Trained coefficients $\hat{\beta}$; predicted CTRs $\hat{y}_i$}

\BlankLine
\textbf{Step 1: Embedding Generation.} Generate text embeddings for each marketing copy:
\[
\phi(t_{k,i}) \leftarrow \text{Embedding}(t_{k,i}).
\]

\textbf{Step 2: Dimensionality Reduction.} Apply PCA to obtain $q$-dimensional treatment representations:
\[
\psi(t_{k,i}) = \Pi \phi(t_{k,i}), \quad \text{where } \Pi \text{ is the PCA projection matrix.}
\]

\textbf{Step 3: Model Training.} Run regression of CTR on the treatment representations with an experiment-level fixed effect term:
\[
\hat{y}_{k,i} = \text{fe}_k + \psi(t_{k,i})^\top \hat{\beta} + \epsilon_{k,i},
\]
where $\text{fe}_k$ are experiment fixed effects.

\textbf{Step 4: Prediction.} For a new experiment, predict relative CTRs:
\[
\hat{y}_i = \psi(t_i)^\top \hat{\beta} + \epsilon_i.
\]

\end{algorithm}

Now we add some discussion regarding Algorithm \ref{alg:ranking}. 

For Step 1, one key choice to make is the embedding method to use. Treatment embeddings can be retrieved from many resources. To name a few, we consider contextual transformers (such as BERT, RoBERTa, Llama), sentence transformers (such as SBERT, MiniLM, OpenAI text embedders). We can even further consider task-finetuned models. In the evaluation part, we compare the performance of the ranking model based on different choices of embeddings to finalize the most suitable choice for our setting. Moreover, we remove mean components when computing treatment embeddings. For example, suppose we are using LLM-based embeddings. We collect a list of common marketing materials $c_i$ with a total number of $N_c$ samples. The final treatment embedding is computed as:
\begin{align}
    \phi(t_i) =\operatorname{Embedding}(t_i)
    - \frac{1}{N_c} \sum_{j=1}^{N_c} \operatorname{Embedding}(c_j)
\end{align}

For Step 2, the performance of PCA involves constructing the projection matrix $\Pi$. Technically, $\Pi$ can be constructed from many sources: training data only, testing data only, and training-testing data altogether. In our case, we found that training on testing data only actually generates the most successful knowledge transfer. 

For Step 3, we essentially applied the principal component regression. Importantly, we added fixed effects to adjust for the experimental level differences for the training dataset. The idea is that some experiments generally have higher CTR than others, so the linear model needs to capture that global difference. 

For Step 4, when we do prediction, we can no longer add the fixed effect term, because for a new experiment, we wouldn't know a priori the average level. Nevertheless, we care more about the relative ranking of the treatment arms instead of the absolute value of the CTR. The prediction here should encode such order information as $\hat{y}_i$ is monotonic in predicted CTRs, yet one should never interpret them as absolute scales of CTR prediction.

     


\section{Semantic attribute characterization} 
In Section \ref{sec:semantic-ranking}, we constructed the ranking model. One important yet untouched question is: how do we interpret the weights of this model? We address this question in this section by connecting the learned model weights with a set of interpretable attributes. 

\subsection{Attribute set and embeddings}
To facilitate a more principled interpretation pipeline, we collect a set of $122$ marketing attributes that reflect the semantic meaning of the marketing contexts. These attributes are finalized through discussion with product managers, marketing executives, and customers. 
Table \ref{tab:attribute-exp} lists a few example attributes. In production, this set of attributes can be dynamically updated by customers based on their own business needs.

\begin{table*}[ht]
    \centering
    \small 
    \caption{Example marketing attributes with description, an example copy, and some representative phrases}
    \label{tab:attribute-exp}
    \begin{tabular}{P{3cm}L{5cm}L{4cm}P{4cm}}
    \toprule 
        \textbf{Attribute Name} & \textbf{Description} & \textbf{Example Copy} & \textbf{Representative Phrases}\\
    \midrule 
        \texttt{action\_oriented} & The \texttt{action\_oriented} motivates customers to act immediately by using direct, time-sensitive language designed to spark engagement and conversions.  & ``Act swiftly—click here to get started today!''  & start now, take action, get started today, \dots  \\ \midrule 
        \texttt{fear\_of\_missing\_out} (\texttt{FOMO}) & The ``\texttt{FOMO}'' attribute uses urgency, exclusivity, and viral appeal to compel customers to
        act immediately and feel part of something popular or limited. & ``Be part of the
        hottest trend everyone’s talking about today!''  & miss, trend, join, popular, \dots
        \\ \midrule 
        \dots & \dots & \dots  & \dots 
        \\ \bottomrule
    \end{tabular}
\end{table*}


  
  
  
    

To enable computation analysis with attributes, we construct embeddings for the attributes and encode the attribute set into a numerical dictionary. Consider $m$ attributes $\mathcal{A} = \{a_1,\dots, a_m\}$. For each attribute, we have a list of representative phrases (see examples in Table \ref{tab:attribute-exp}) $\{u_{m,i}\}_{i = 1}^{N_{a_m}}$ with a total number of $N_{a_m}$ phrases for attribute $a_m$. The attribute embedding is computed as:
\begin{align}
    v(a_m) =& \frac{1}{N_{a_m}}\sum_{i=1}^{N_{a_m}}\operatorname{Embedding}(u_{m,i}) \\
    &- \frac{1}{\sum_{m=1}^M N_{a_m}}\sum_{m=1}^M\sum_{i=1}^{N_{a_m}}\operatorname{Embedding}(u_{m,i}).
\end{align}
Importantly, here we applied a similar demeaning approach as in the treatment embedding generation step in Section \ref{sec:semantic-ranking}. The final collection of embeddings is saved as a matrix $V$, which we call the attribute embedding dictionary. 

\subsection{Attribute score}
Attribute scores represent the level of marketing attributes inside a treatment. The attribute scores $s_{k, i, a}$ for treatment $t_{k, i}$ in terms of attribute $a$ is given as follows:
\begin{align}
    s_{k, i, a} = \phi(t_{k, i})^\top v(a), \quad a \in \mathcal{A}. 
\end{align}
where $v(a)$ is the $p$-dimensional vector as the embedding of attribute $a$. Written in a more compact form:
\begin{align}
    s_{k, i} = V \phi(t_{k, i}),
\end{align}
where $s_{k, i}$ is the vector of attribute scores for treatment $t_{k, i}$. 

\begin{figure}
    \centering
    \includegraphics[width=0.8\linewidth]{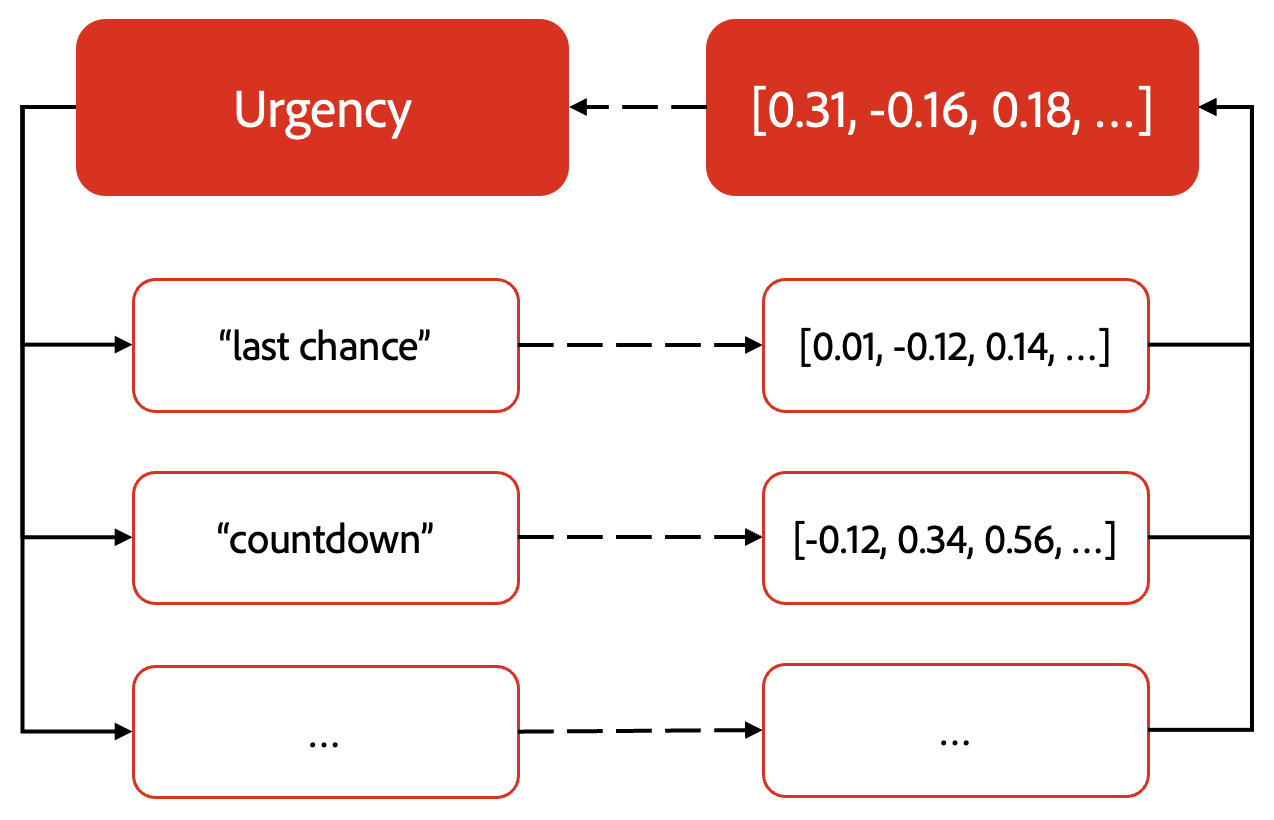}
    \caption{Attribute characterization}
    \label{fig:attribute}
\end{figure}

\subsection{Impact coefficient}
The attribute score can be combined with the ranking model weights to quantitatively measure the importance of each attribute. We call this index the \textit{impact coefficient}. Below in Algorithm \ref{alg:impact-coef}, we summarize the pipeline of computing this coefficient.
\begin{algorithm}[ht]
\small 
\caption{Attribute-Space Decomposition and Regularized Ranker Estimation}
\label{alg:impact-coef}
\SetAlgoLined
\KwIn{PCA projection matrix $\Pi$, attribute embedding dictionary $V$, learned PCA-ranker coefficients $\hat{\beta}$, embeddings $\phi(t_i)$}
\KwOut{Attribute-space coefficients $\hat{\beta}''$}

\BlankLine
\textbf{Step 1: Reparameterize PCA-Ranker Model.} \\
Compute the equivalent coefficient in the original embedding space by reversing the PCA step:
\begin{align}
\hat{\beta}' \leftarrow \Pi^\top  \hat{\beta}, \qquad 
\hat{y}_i = \phi(t_i)^\top \hat{\beta}'. \label{eq:ranker_model}
\end{align}

\textbf{Step 2: Decompose Embeddings into Attribute Embeddings plus Nuisance.} \\
Project each $\phi(t_i)$ onto the attribute space and its orthogonal complement:
\begin{align}
\phi(t_i) = & V^\top (VV^\top)^{-1} V \phi(t_i) + (I - V^\top (VV^\top)^{-1} V)\phi(t_i) \\ 
           = & V^\top (VV^\top)^{-1}s_i + \tilde{\phi}_i.
\end{align}
Hereafter we use $V^+ = V^\top (VV^\top)^{-1}$ to denote the pseudo inverse of $V$.

\textbf{Step 3: Combine Representations.} \\
Substitute decomposition into the ranker model:
\begin{align}\label{eqn:pred}
\hat{y}_i = s_i^\top V^{+\top}\hat{\beta}' + \tilde{\phi}_i^\top \hat{\beta}'
           \triangleq s_i^\top \hat{\beta}'' + \tilde{\phi}_i^\top \hat{\beta}',
\end{align}
with $\hat{\beta}'' = V^{+\top}\hat{\beta}'$.

\textbf{Step 4: Regularized Attribute Coefficient Estimation.} \\
Estimate $\hat{\beta}''$ via a sign-constrained Lasso problem:
\begin{align}
\hat{\beta}'' = & \arg\min_{\beta''}
    \|\hat{\beta}' - V^\top \beta''\|_2^2 + \lambda\|\beta''\|_1, \label{eq:ranker_model_decomp}
    \\
    & \text{s.t. } \beta'' \cdot \text{sign}(V^\top\hat{\beta}') \ge 0,
\end{align}
where $\lambda$ controls sparsity and is selected by $K$-fold cross-validation.
\end{algorithm}

We add some explanations. Step 4 is important. For the coefficient $\hat\beta''$ in definition \eqref{eqn:pred}, the following lemma suggests that it can be re-expressed as the solution of a least squares problem:
\begin{lemma}[Re-expressing $\hat\beta''$]\label{lem:re-expression}
For the impact coefficient $\hat\beta''$ in definition \eqref{eqn:pred}, we have
\begin{align}\label{eqn:beta-dprime-ls}
    \hat{\beta}'' = \arg \min_{\beta''} \|\hat{\beta}' - V^\top \beta''\|^2_2.
\end{align}    
\end{lemma}

This motivates us to further add regularization to the least square problem to reduce collinearity and also impose a sign constraint \cite{meinshausen2013sign} for business impact consistency:
{\small 
\begin{align}
    \hat{\beta}'' = \arg \min_{\beta''} \|\hat{\beta}' - V^\top \beta''\|^2_2 + \lambda \|\beta''\|_1,~
    \text{s.t. } \beta'' \cdot \text{sign}(V^\top\hat{\beta}') \geq 0.
\end{align}
}
The constraint guarantees that $\hat{\beta}''_a$ always has the same sign as $(V \hat{\beta}')_a$ (or is exactly zero), ensuring that ``positive/negative attributes'' are interpreted in the same direction regardless of how many new attributes are added. Besides, this sign-constrained Lasso is a convex problem and can be solved via existing implementations in Python, for example, \texttt{sklearn.linear\_model.Lasso} with argument \texttt{positive = TRUE}. 

\section{Downstream application in production}    

In this section, we present several downstream applications of the above semantic ranking model, which we have already deployed in production.
These applications mainly involve building several crucial indices for measuring the quality of the experiments from different aspects.


\subsection{Insights generation}\label{sec:insights-generation}

If an experiment has finished running and reported a significant CTR difference, the marketers might want to get a more comprehensive understanding and summarization of why some treatment arms are better than others. We call such a goal ``Insights Generation'' for completed experiments. Our solution is to build an experiment-level index (which we call the insight index) that can quantitatively measure how much an attribute is contributing to the difference between treatment arms, and what is the direction of that contribution. 

To build this index, suppose two treatments $t_i,t_j$ have attribute score difference 
\begin{align}
    \Delta_{i,j} s_a = s_{i,a} - s_{j,a}, \quad a\in\mathcal{A},    
\end{align}
which measures the difference of the level of expressiveness between treatment arm  $t_i$ and $t_j$, in terms of the attribute set. To further understand how much this difference contributes to the CTR differences, we resort to the prediction step \eqref{eqn:pred} in Algorithm \ref{alg:ranking}, and express the predicted CTR difference as:
\begin{align}
    \Delta_{i,j}\hat{y} &= (\Delta_{i,j} \vec{s})^\top V^{+\top}\hat{\beta}' + (\Delta_{i,j}\tilde{\phi})^\top \hat{\beta}'
    \triangleq (\Delta_{i,j} \vec{s})^\top \hat{\beta}'' + (\Delta_{i,j}\tilde{\phi})^\top \hat{\beta}'.
\end{align}
From the above equation, we can see that, for attribute $a$, its difference in arms $t_i$ and $t_j$ will lead to a difference of $\Delta_{i,j} s_a \cdot \hat{\beta}''_a$ in the predicted CTR. Intuitively, this quantity takes the impact coefficients into account, which reweight the importance of $\Delta_{i,j} s_a$. Due to this interpretation, we define this special quantity as \textit{insight index} for attribute $a$ and write this as
\begin{align}\label{eqn:pairwise_contribution}
    R^{\textup{Ins}}_a = \Delta_{i,j} s_a \cdot \hat{\beta}''_a.
\end{align}
The sign of $R^{\textup{Ins}}_a$ now indicates the direction of the contribution of attribute $a$ to the predicted relative CTR difference $\Delta_{i,j} \hat{y}$. The magnitude of $R^{\textup{Ins}}_a$ indicates the strength of the contribution of attribute $a$.

With the insight index, we can perform downstream analyses to derive insights, such as plotting the distribution of $\Delta_{i,j} s_a \cdot \hat{\beta}''_a$ for each attribute $a$, or finding the top $k$ positive and negative attributes and hinting at their impacts on customers. For example, combining with LLM, we have the example insight extraction output in Box \ref{box:insight}. 

    
{\small 
\begin{mybox}[floatplacement=ht,label={box:insight}]{red}{Insight extraction example}
  \texttt{\textbf{Surprising Statistic helps}: Top performing treatments like "Create incredible art and make complex edits in seconds" leverage unexpected speed and capability, capturing attention and driving engagement; lower-ranked treatments lack such striking claims, resulting in less intrigue and diminished appeal.}
    
\end{mybox}
}

\subsection{Opportunity generation}
\label{sec:opportunity-generation}

We want to provide opportunities for experiments based on the ranker model. Basically, we want to compute the attribute scores for each treatment in terms of a set of marketing attributes and combine this with the ranking model weights to generate opportunities for a new experiment. Similar to insight generation, the key step here is to construct the opportunity index. The key idea is that, we search for the important attributes that are missing in the treatments being tested in a new experiment. We construct a combined index to measure both the importance and the level of expression of the attributes in one experiment.

To measure importance, we take the rank $R_a^{\text{Imp}}$ of the impact coefficients $\beta''_a$:
\begin{align}
    R_a^{\text{Imp}} = \textup{Rank}(\beta_a'').
\end{align}
A smaller $R_a^{\text{Imp}}$ means that attribute $a$ is of higher importance. To measure the level of expression of an attribute, we take the rank $R_a^{\text{Exp}}$ of the mean attribute scores:
\begin{align}
    R_a^{\text{Exp}} = \textup{Rank}\left(\max_{i\in\{1,\dots,I\}} s_{i,a} \right). 
\end{align}
A smaller $R_a^{\text{Exp}}$ indicates attribute $a$ is more explored in the variants. Now we build an opportunity index to combine these two metrics:
\begin{align}
    R^{\text{Opp}}_a = R_a^{\text{Imp}} - R_a^{\text{Exp}}.
\end{align}
A smaller $R^{\text{Opp}}_a$ means attribute $a$ has higher importance and is less represented by the current variants. This motivates a direct application: we can take several attributes that have the smallest $R^{\text{Opp}}_a$, and suggest an improvement in terms of these attributes in future iterations of the experiments. An example of opportunity generation is given by Box \ref{box:opportunity}. 
{\small 
\begin{mybox}[floatplacement=ht,label={box:opportunity}]{brown}{Opportunity generation example}
  \texttt{\textbf{Unified Brand Voice}: Ensures messaging aligns with brand identity for coherence. Example: "Your trusted partner in everyday savings." }
    
    \texttt{\textbf{Social Proof}: Builds trust by showing others' positive experiences. Example: "Over 10,000 satisfied customers can’t be wrong!"}
\end{mybox}
}
{\small 
\begin{mybox}[floatplacement=ht,label={box:opportunity}]{green}{Conversion and learning potential example}
  \textbf{Suggestion}: Increasing \texttt{Unified Brand Voice} 

  \textbf{Conversion potential}: Medium 
  
  $\star$ historical experiment suggests medium impact in rank prediction for this attribute

  \textbf{Learning potential}: High 
  
  $\star$ this attribute is not adequately explored in both current and historical experiments.
\end{mybox}
} 
\subsection{Impact and Novelty quantification}
In this section, we introduce the insight index and the opportunity index, which are synthesized measures combining attribute scores and ranking model weights. 

Based on the ranking model, we have pointed out the crucial role of the impact coefficients $\vec{\beta}''$. We propose to bin the impact coefficients, for example, into several levels (e.g., \texttt{Strong} vs. \texttt{Medium} vs. \texttt{Weak}), such that we can report an impact measure for the selected attributes in driving the conversion.

Based on the attribute score, it is useful to measure how novel an attribute is for the customer. In other words, for an attribute, we can measure how much it is expressed in the current experiments as well as in the historical experiments. If both are low, it means that this attribute has not fully explored by the current and past experiments. To quantitatively measure the novelty, let $\bar{s}_{a,\textup{history}}$ be the average attribute score for attribute $a$ across a set of historical experiments. We can define the following Novelty Index: 
\begin{align}
    R^{\text{Novel}}_a & = \textup{Rank}(\max_{i\in\{1,\dots,I\}} s_{i,a}) + \textup{Rank}(\bar{s}_{a,\textup{history}}) \\
    & = R_a^{\text{Exp}}
    +
    \textup{Rank}(\bar{s}_{a,\textup{history}}). 
\end{align}
A small novelty index indicates that attribute $a$ is less explored and potentially contains many opportunities. 

In our Experimentation Agent, we implemented the impact measure and novelty measure for the Opportunity Module, which provide informative quantitative check for the experiments.

\section{Text Generation with LLMs}
\label{sec:llm-part}

\subsection{Insights}
This component transforms the quantitative attribution in Sections~\ref{sec:semantic-ranking}-\ref{sec:insights-generation} into concise, text-grounded explanations. It has three stages: attribute selection, prompt-based explanation, and quality filtering.

\textbf{Attribute selection.} Given the impact coefficients $\hat{\beta}''$ in the attribute space (Section~\ref{sec:insights-generation}) and the per-treatment attribute scores $s_i$, we form a contrast between the empirically best and worst treatments by outcome (CTR): $\Delta s = s_{\text{best}} - s_{\text{worst}}$. We then compute per-attribute contributions $c_a = (\Delta s)_a\, \hat{\beta}''_a$. We apply a positivity filter and retain only attributes with strictly positive contribution ($c_a>0$), rank the remaining attributes by $c_a$ in descending order, and select the top-$k$. For presentation to the LLM, each attribute’s effect sign is taken from $\mathrm{sign}(\hat{\beta}''_a)$ (``positive'' if $\hat{\beta}''_a>0$, otherwise ``negative''). This follows Eq.~(\ref{eqn:pred}): signed score differences, scaled by the learned coefficient, quantify directional lift from best over worst.

\textbf{Prompting and output format.} We serialize each experiment as an ordered list of treatments (best\,$\to$\,worst by CTR) with compact text deltas, and pass the selected attributes together with their effect polarity (from $\mathrm{sign}(\hat{\beta}''_a)$). The prompt adopts established prompt-engineering techniques to elicit concise, text-grounded rationales: (i) role conditioning and explicit task constraints, (ii) few-shot exemplars to demonstrate the desired style, (iii) structured output requirements to enforce a fixed number of items and short, evidential sentences, and (iv) citation-first grounding that requires quoting phrases from the provided treatments. This design prioritizes determinism, scannability, and faithfulness to observable text over free-form priors. The template and constraints are tuned with the help of an LLM-based prompt designer using a held-out acceptance signal (Section~\ref{sec:o3-eval}).

\textbf{Post-processing and self-reflection.} We apply lightweight normalization (removing auxiliary markers, collapsing whitespace) and use a rule-based parser to produce a mapping from attribute to explanation. For quality control, a per-attribute self-reflection step prompts the model to critique each candidate against a concise rubric (faithful citations to the provided treatments, correct attribution, concise one-paragraph rationale) and to emit a schema-constrained acceptance decision, only accepted candidates are retained. This internal self-reflection filter is distinct from the external evaluation in Section~\ref{sec:o3-eval}. The self-reflection prompt uses standard prompt-engineering elements: role conditioning, explicit criteria, few-shot scoring examples, and fixed-format outputs, and its constraints are tuned with a held-out acceptance signal, so the final number of returned insights may be fewer than $k$ when low-quality candidates are rejected.

\subsection{Opportunity}
This component proposes experiment-level opportunities by identifying high-impact yet under-expressed attributes and generating concise, illustrative suggestions to guide future iterations.

\textbf{Selecting missing high-impact attributes.} Starting from the global impact coefficients $\hat{\beta}''$, we keep only strictly positive-impact attributes (strengthening the attribute increases predicted CTR). We optionally apply an absolute-impact threshold (relative to the maximum absolute impact) to drop negligible attributes. For the current experiment, we measure each attribute’s degree of exploration (e.g., maximum observed score across treatments) and form two ranks: (i) importance and (ii) expression. The opportunity score is the rank difference (importance minus expression), and attributes with high importance but low expression are selected. When needed, we fall back to the highest positive-impact attributes. This procedure instantiates the index in Section~\ref{sec:opportunity-generation}.

\textbf{Prompting and output format.} The LLM receives (i) the experiment’s treatment texts and (ii) the selected missing attributes. At a high level, the prompt uses role conditioning, a few illustrative examples, and simple structure constraints to elicit concise, well-structured suggestions grounded in the provided text.

\textbf{Parsing and enrichment.} We apply a rule-based parse of the structured output to enforce a consistent format and drop malformed blocks to maintain quality. We also enrich each opportunity with: (i) \emph{learning potential}, reflecting novelty relative to the experiment and historical data, and (ii) \emph{conversion potential}, reflecting the relative magnitude of the learned impact coefficient. Both are reported at a coarse level.

\section{Post-launch evaluation}
\label{sec:evaluation}
\subsection{Transfer from the Upworthy dataset}
In this section, we evaluate the transfer learning results based on Algorithm \ref{alg:ranking}, on a set of real $N_{\text{test}} = 65$ Adobe customer experiments. Each experiment contains two or more arms, and has been completed and confirmed to deliver a statistically significant difference among the treatment arms. In other words, the rank of the arms has been verified by data. For experiment $i$, denote the rank for variant $m$ as $R_{i,m}$, and the predicted rank based on Algorithm \ref{alg:ranking} as $\hat{R}_{i,m}$. We consider two metrics: (i) Mean Spearman's rank correlation coefficient, which is given by
\begin{align*}
    \rho = \frac{1}{N_{\text{test}}}\sum_{i=1}^{N_{\text{test}}} \rho_i, \quad \text{ where }
    \rho_i = \frac{\textup{Cov}({R}_{i,\cdot}, {\hat{R}}_{i,\cdot})}{\sqrt{\text{Var}({R}_{i,\cdot}) \text{Var}({\hat{R}}_{i,\cdot})}}.  
\end{align*}
This measures the rank prediction consistency for Algorithm \ref{alg:ranking}. 
(ii) Top-1 accuracy, which checks the portion of experiments that achieves perfect rank-1 prediction: let $m^\star_i$ and $\hat{m}_i$ be the true and estimated best variant in experiment $i$, respectively, we define 
\begin{align*}
    \text{Top-1 Accuracy} = \frac{1}{{N_{\text{test}}}}\sum_{i=1}^{{N_{\text{test}}}}1(m^\star_i = \hat{m}_i) 
\end{align*}

The final results are reported in Table \ref{tab:perf}, together with the confidence interval based on the standard deviation of the metrics across the validation experiments. The transfer learning achieves a pretty significant result, both in terms of general rank correlation and top-1 accuracy, compared with the original baseline where markets make a random guess and do not rely on the prediction model to make the decision. 
\begin{table}[ht!]
\centering
\small 
\caption{Performance comparison.}
\label{tab:perf}
\begin{tabular}{lcc}
\toprule
\textbf{Metric} & $\boldsymbol{\rho}$ (rank corr.) & \textbf{Top-1 Acc.} \\
\midrule
Transfer Learning & $0.5144 \pm 0.1658$ & $0.7021 \pm 0.1180$ \\
Random Guess Baseline                  & $0$                  & $0.4293$            \\
\bottomrule
\end{tabular}
\end{table}

\subsection{Leave-One-Out on the choice of embeddings}
In this section, we report some numerical results regarding the choice of embeddings for building the prediction model. 
We compare three embedding families: (i) MiniLM \citep{wang2020MINILM}, a compact Transformer encoder widely used for sentence-level representations; (ii) Llama embeddings \citep{grattafiori2024llama3herdmodels}, extracted from a modern LLM encoder to capture richer, long-range semantics; and (iii) Attribute Score, a LLM-as-a-judge approach to directly score the variants in terms of the attribute set we introduced, in a scale of $1$ to $5$. We report Spearman's rank correlation between the predicted ranking and the observed hold-out outcomes, averaged across leave-one-out folds, with confidence intervals based on standard deviations across metrics on the hold-out experiments. We again include a random guess baseline, which yields zero expected correlation.

Table~\ref{tab:spearman_ajo} summarizes the leave-one-out performance. Llama model achieves the highest average correlation ($0.727 \pm 0.116$), indicating stronger alignment between its semantic geometry and the outcome-relevant signal. 
\begin{table}[ht!]
\centering
\small 
\caption{Spearman rank correlations on the hold-out datasets.}
\label{tab:spearman_ajo}
\begin{tabular}{lcc}
\toprule
\textbf{Embedding} & \textbf{Hold-out Experiments} & \textbf{Random Guess} \\
\midrule
MiniLM & $0.454 \pm 0.155$ & $0$ \\
Llama  & $0.727 \pm 0.116$ & $0$ \\
Attribute Score  & $0.576 \pm 0.142$ & $0$ \\
\bottomrule
\end{tabular}
\end{table}


\subsection{Insights generation and evaluation}

In this section, we evaluate the quality of automatically generated insights about creative variants and experimental outcomes. We generate insights following the proposed framework on a set of Adobe customer experiments. Then we perform human annotation and report the acceptance rate. Besides, we also report LLM (GPT-o3) evaluation results in Appendix \ref{sec:o3-eval}. 

\paragraph{Human evaluation.}
Human evaluation results are reported in Table~\ref{tab:llm_insight_quality}. We asked experienced marketers and engineers to assess whether each surfaced insight is well-structured, accurate, and reasonable. Under this protocol, GPT-4o produced the most accepted insights overall (46 accepted out of 52 generated; 88.46\%), while LLaMA-70B showed a similarly high acceptance rate (88.46\%) but generated far fewer total insights (23 accepted out of 26). GPT-4o-mini generated substantially more candidates but with much lower acceptance (42.04\% and 33.33\%, respectively), indicating noisier outputs despite higher throughput.

\begin{table}[h]
\centering
\small
\caption{Comparison of insight generation and quality across different language models, in terms of the total numbers of: (1) experiments reviewed, (2) experiments with insights, (3) insights generated, (4) high-quality insights. }
\label{tab:llm_insight_quality}
\begin{tabular}{>{\raggedright\arraybackslash\scriptsize}p{3cm}ccc}
\toprule
 & \textbf{GPT-4o} & \textbf{GPT-4o-mini} & \textbf{LLaMA-70B}  \\
\midrule
\textbf{Experiments Reviewed}       & 65 & 65 & 65 \\
\textbf{Experiments with Insights}  & 37 & 60 & 22  \\
\textbf{Insights Generated}         & 52 & 157 & 26 \\ 
\textbf{High-Quality Insights} & 46 & 66 & 23 \\ 
\midrule 
\textbf{Acceptance Rate (\%)}                 & 88.46 & 42.04 & 88.46 \\
\bottomrule
\end{tabular}
\end{table}

\subsection{Opportunity generation evaluation}

We evaluate the quality of automatically generated opportunities produced by different language models. We report both human acceptance rates in the main content and LLM (GPT-o3) evaluation results in Appendix \ref{sec:o3-eval}. 

\paragraph{Human evaluation.}
Table~\ref{tab:llm_opportunity_quality} summarizes human annotations. GPT-4o and GPT-4o-mini yielded the largest number of accepted opportunities at similar high acceptance (86.15\% and 84.10\%), reflecting strong coverage when three proposals are made per experiment. LLaMA-70B shows slightly lower human acceptance (82.11\%) and, due to annotation budget, was assessed on 41 experiments. 



\begin{table}[ht!]
\centering
\small
\caption{Comparison of opportunity generation and quality across different language models, in terms of the total numbers of: (1) experiments reviewed, (2) experiments with opportunities, (3) opportunities generated, (4) high-quality opportunities. }
\label{tab:llm_opportunity_quality}
\begin{tabular}{>{\raggedright\arraybackslash\scriptsize}p{3.1cm}ccc}
\toprule
 & \textbf{GPT-4o} & \textbf{GPT-4o-mini} & \textbf{LLaMA-70B} \\
\midrule
\textbf{Experiments Reviewed}                & 65 & 65 & 41  \\
{\scriptsize \textbf{Experiments with Opportunities}}                & 65 & 65 & 41  \\
\textbf{Opportunities Generated}    & 195 & 195 & 123 \\
\textbf{High-Quality Opportunities} & 168 & 164 & 101  \\
\midrule 
\textbf{Acceptance Rate (\%)}                      & 86.15 & 84.10 & 82.11 \\
\bottomrule
\end{tabular}
\end{table}

\section{Discussion}

Experimentation Accelerator helps teams use scarce traffic wisely by ranking variants before testing, explaining wins through a simple attribute view, and turning gaps into concrete creative suggestions. 
In the next steps, we plan to extend the approach beyond text to image/video, audio, and page layouts, and to outcomes like conversions or revenue.  
We want to document potential limitations of the pipeline for broader use cases: (i) transfer learning is an important step to check when there is a domain shift from Upworthy to customer data. While we witness this success in our case, such transferability should always be carefully verified for new customers. (ii) Language coverage for the current approach is limited to English. Generalization to other languages should be carefully validated. (iii) The expressivity of embeddings is crucial for the pipeline, yet the performance varies highly across different text embedders based on our observations. In real production, it is also important to check a variety of embedding models to finalize the best one for a particular group of target customers.


\bibliographystyle{ACM-Reference-Format}
\bibliography{references}

\appendix

\section{Proof of Lemma \ref{lem:re-expression}}
\begin{proof}[Proof of Lemma \ref{lem:re-expression}]
    Based on Definition \ref{eqn:pred}, $\hat\beta''$ is expressed as:
    \begin{align}
        \hat\beta'' = V^{+\top} \hat\beta',
    \end{align}
    where $V^{+\top}$ is the Penrose-Moore pseudo inverse. Checking the least squares problem \eqref{eqn:beta-dprime-ls}, we know that the first order optimality condition gives that
    \begin{align}
        VV^\top \hat\beta'' = V\hat\beta'.
    \end{align}
    This further leads to 
    \begin{align}
        \hat\beta'' = (VV^\top)^{-1} V\hat\beta' = V^{+\top} \hat\beta'.
    \end{align}
\end{proof}


\section{LLM-as-a-judge for automated evaluation}
\label{sec:o3-eval}
We use an external LLM judge to validate both insight explanations and opportunity suggestions. The judge only sees the serialized treatment content and the candidate output, and returns a binary decision (acceptable=1, problematic=0) with a brief justification.

\textbf{Prompting and rubric.} A compact prompt provides one worked example and a minimal rubric. For insights, the judge checks that the explanation correctly accounts for performance differences and quotes phrases from the treatments. For opportunities, it checks that the suggestion identifies a truly missing, high-value attribute, explains why it would help, and includes one to two short, illustrative examples.

\textbf{Model and parsing.} Judging uses a reasoning model ``o3''. Responses follow a simple per-item schema, we extract the score and justification with a rule-based parser.

\paragraph{LLM evaluation (GPT-o3 judge).}
To complement human review at scale, we use GPT-o3 as an LLM judge with the same rubric. Table~\ref{tab:llm_insight_eval_grouped} summarizes means and standard deviations over the GPT-o3-judging quality for multiple rounds of insights generation across all experiments. The judge is stricter than humans (lower acceptance across the board) but preserves the relative ordering: GPT-4o and LLaMA-70B attain the highest acceptance rates, with GPT-4o yielding markedly more accepted insights in total due to higher generation volume. Smaller models (GPT-4o-mini, LLaMA-8B) generate many candidates, but few survive the stricter filter.

\begin{table*}[ht!]
\centering
\caption{Insight generation evaluation with GPT-o3 as a judge.}
\label{tab:llm_insight_eval_grouped}
\begin{tabular}{cccccc}
\toprule
\multirow{2}{*}{\textbf{Language Model}} &
\multicolumn{2}{c}{\textbf{Number of Experiments}} &
\multicolumn{2}{c}{\textbf{Number of Insights}} &
\textbf{Acceptance Rate (\%)} \\
\cmidrule(lr){2-3}\cmidrule(lr){4-5}
& \textbf{Reviewed} & \textbf{with Insight} & \textbf{Generated} & \textbf{With Good Quality} & \\
\midrule
GPT-4o        & 65 & $45.250 \pm 0.986$ & $49.250 \pm 2.511$ & $26.250 \pm 2.421$ & $53.87 \pm 6.70$ \\
GPT-4o-mini   & 65 & $61.000 \pm 0.471$ & $145.250 \pm 3.140$ & $24.250 \pm 1.590$ & $16.87 \pm 1.35$ \\
LLaMA-70B     & 65 & $23.000 \pm 0.000$ & $27.750 \pm 0.553$ & $16.500 \pm 1.106$ & $59.53 \pm 4.34$ \\
LLaMA-8B      & 65 & $53.750 \pm 0.289$ & $93.000 \pm 2.550$ & $18.000 \pm 2.550$ & $19.75 \pm 3.27$ \\
\bottomrule
\end{tabular}
\end{table*}

\paragraph{LLM evaluation (GPT-o3 judge) for opportunity generation.}
To scale evaluation, we also use GPT-o3 as an automatic judge with the same rubric. Results in Table~\ref{tab:llm_opportunity_eval_grouped} report means and standard deviations over the GPT-o3-judging quality for multiple rounds of opportunity generation across all experiments.. Acceptance rates are uniformly high, and the ordering broadly aligns with human review: GPT-4o and LLaMA-70B are competitive at the top, while GPT-4o-mini and LLaMA-8B perform slightly worse. Note that LLaMA-70B and LLaMA-8B generated 130 opportunities (two per experiment on average) in this configuration, whereas GPT-4o variants produced the full 195 (three per experiment), affecting total accepted count even when acceptance rates are close.

\begin{table*}[ht!]
\centering
\caption{Opportunity generation evaluation with GPT-o3 as a judge.}
\label{tab:llm_opportunity_eval_grouped}
\begin{tabular}{cccccc}
\toprule
\multirow{2}{*}{\textbf{Language Model}} &
\multirow{2}{*}{\textbf{\# Experiments}} &
\multicolumn{2}{c}{\textbf{Number of Opportunities}} &
\multirow{2}{*}{\textbf{Acceptance Rate (\%)}} \\
\cmidrule(lr){3-4}
& \textbf{Reviewed} & \textbf{Generated} & \textbf{With Good Quality} & \\
\midrule
GPT-4o        & 65 & $195.000 \pm 0.000$ & $158.250 \pm 1.658$ & $81.15 \pm 0.85$ \\
GPT-4o-mini   & 65 & $195.000 \pm 0.000$ & $147.250 \pm 1.443$ & $75.51 \pm 0.74$ \\
LLaMA-70B     & 65 & $130.000 \pm 0.000$ & $104.250 \pm 3.032$ & $80.19 \pm 2.33$ \\
LLaMA-8B      & 65 & $130.000 \pm 0.000$ & $101.500 \pm 1.202$ & $78.08 \pm 0.92$ \\
\bottomrule
\end{tabular}

\end{table*}

\end{document}